\documentclass{article}

\usepackage{arxiv}

\usepackage[utf8]{inputenc} 
\usepackage[T1]{fontenc}    
\usepackage{hyperref}       
\usepackage{url}            
\usepackage{booktabs}       
\usepackage{amsfonts}       
\usepackage{nicefrac}       
\usepackage{microtype}      
\usepackage{lipsum}		
\usepackage{graphicx}
\usepackage{natbib}
\usepackage{doi}
\usepackage{amssymb}
\usepackage{enumitem}
\usepackage{mathtools}
\usepackage{amsthm}
\newtheorem{definition}{Definition}
\newtheorem{remark}{Remark}
\newtheorem{proposition}{Proposition}
\newtheorem{assumption}{Assumption}
\newtheorem{theorem}{Theorem}
\newtheorem{corollary}{Corollary}
\newtheorem{lemma}{Lemma}
\usepackage{tabularx}

\title{Redundancy as a Structural Information Principle for Learning and Generalization}

\author{
  Yuda Bi \\
  Tri-Institutional Center for Translational Research in Neuroimaging and Data Science (TReNDS)\\
  Georgia State University, Georgia Tech, and Emory University \\
  Atlanta, GA, USA \\
  \texttt{ybi3@gsu.edu} \\
  \And
  Ying Zhu \\
  Wenzhou Institute, University of Chinese Academy of Sciences (UCAS) \\
  Wenzhou, Zhejiang, China \\
  \texttt{joying0703@163.com} \\
  \And
  Vince D.~Calhoun \\
  Tri-Institutional Center for Translational Research in Neuroimaging and Data Science (TReNDS)\\
  Georgia State University, Georgia Tech, and Emory University \\
  Atlanta, GA, USA \\
  \texttt{vcalhoun@gsu.edu} \\
}

\renewcommand{\shorttitle}{\textit{arXiv} Template}

\hypersetup{
pdftitle={A template for the arxiv style},
pdfsubject={q-bio.NC, q-bio.QM},
pdfauthor={David S.~Hippocampus, Elias D.~Striatum},
pdfkeywords={First keyword, Second keyword, More},
}

\begin{document}
\maketitle

\begin{abstract}
	We present a theoretical paradigm that extends classical information theory to finite and structured systems by redefining \emph{redundancy} as a fundamental quantity of information organization rather than inefficiency.  
Within an $f$-divergence framework, redundancy is formalized as $\mathcal{R}_{f}(X) = D_{f}(P_X \| \Pi_X)
= \mathbb{E}_{\Pi_X}\!\big[f\!\big(\tfrac{p(x)}{\prod_i p_i(x_i)}\big)\big]$,  
where $p(x)$ is the joint density of $(X_1,\dots,X_n)$,  
$p_i(x_i)$ their marginals, and $f$ a convex kernel defining the geometry of informational dependence.  
Different choices of $f$ recover mutual information, $\chi^2$ redundancy, and spectral redundancy as special cases, unifying diverse notions under a single mathematical principle.  
This reveals that classical measures are not isolated heuristics but projections of a single redundancy geometry.  
The framework shows that redundancy is bounded both above and below, yielding a natural equilibrium $R^{*}$ between over-compression (loss of structure) and over-coupling (collapse).  
In contrast to the asymptotic regime where minimizing redundancy optimizes transmission efficiency,  
finite, structured systems—where real-world learning operates—achieve maximal stability and generalization near this equilibrium. Thus, redundancy emerges as a \emph{structural information principle}: a self-organizing property that governs how information is coherently structured rather than transmitted.  
Experiments with masked autoencoders (MAE) serve to \emph{verify and visualize} the theory rather than pursue performance benchmarks.  
They confirm the predicted equilibrium $R^{*}$, where latent redundancy stabilizes and generalization peaks.  
Together, these results establish redundancy as a measurable and tunable quantity bridging the asymptotic world of communication and the finite world of learning.
\end{abstract}

\keywords{Information Theory \and Redundancy \and Generalization \and Representation Learning \and Structural Information Principle}

\section{Introduction}
Redundancy is one of the oldest yet least unified concepts in the quantitative sciences.  
From Shannon’s original formulation of information theory in the 1940s, redundancy has been viewed primarily as an inefficiency—a measure of wasted bits relative to the entropy limit of a code \cite{shannon1948mathematical}.  
In neuroscience and sensory coding, however, the term acquired an almost opposite meaning: redundant neural firing or overlapping receptive fields were interpreted as signatures of robustness and fault tolerance in biological systems \cite{barlow2001redundancy, narayanan2005redundancy, tononi1994measure} with methods often designed to leverage redundancy explicitly \cite{kazemivash2025st}. 
In statistics and machine learning, redundancy appears again under diverse guises—multicollinearity in regression \cite{gunst1975regression}, overparameterization in neural networks \cite{zhang2016understanding, belkin2019reconciling}, and correlated features in representation learning—each carrying context-dependent interpretations of inefficiency, degeneracy, or regularization \cite{li2023exploiting, berchenko2024simplicity, allen2019convergence, chan2022mitigating}.  
In physics and complex systems, redundant degrees of freedom are sometimes regarded as the very basis of stability and emergence, echoing the notion that structure arises from constrained repetitions rather than minimal codes \cite{may1972will, lloyd2007programming, prigogine1982being}.  
Despite its ubiquity, these perspectives remain largely fragmented, and no single mathematical framework consistently links redundancy across these domains. Historically, classical information theory and signal processing treated redundancy as something to be eliminated.  
In the Shannon framework, an optimal code is one that minimizes redundancy, achieving the highest information rate per transmitted symbol \cite{shannon1948mathematical}.  
Source coding theorems and compression algorithms were built precisely on this ideal—removing repeated or correlated patterns to approach the entropy limit of a signal \cite{huffman2007method, ziv2003universal}.  
In statistics and data representation, the same logic underlies many foundational techniques:  
principal component analysis (PCA) seeks orthogonal directions of maximal variance to eliminate correlated features;  
independent component analysis (ICA) explicitly minimizes statistical dependence between latent sources \cite{hyvarinen2000independent, bell1997independent, tishby2000information};  
and sparse or efficient coding models in neuroscience posit that sensory neurons should represent the world with minimal overlap or redundancy among firing patterns \cite{olshausen1996emergence}.  
Across these traditions, redundancy has been equated with inefficiency, overfitting, or wasted capacity—a byproduct to be compressed away rather than a structural property to be understood.

Yet in finite, noisy, and structured regimes, the classical view of redundancy as inefficiency breaks down.  
We propose a paradigm shift that extends Shannon’s information theory from asymptotic coding efficiency to finite-sample informational organization.  
In this new framework, redundancy is not the waste of bits, but the \emph{geometry of informational dependence}—a structural degree of freedom through which data organize meaning, stability, and generalization.

Formally, we define redundancy as an $f$-divergence from statistical independence,  
$\mathcal{R}_{f}(X) = D_{f}(P_X \| \Pi_X)
= \mathbb{E}_{\Pi_X}\!\big[f\!\big(\tfrac{p(x)}{\prod_i p_i(x_i)}\big)\big]$,  
where $f$ specifies the geometric kernel of informational coupling.  
This single functional family $\{\mathcal{R}_f\}$ unifies diverse notions of redundancy across disciplines as its different projections:  
in \emph{information theory}, $\mathcal{R}_{\mathrm{KL}}$ corresponds to mutual information and total correlation (entropy-based projection);  
in \emph{statistics}, $\mathcal{R}_{\chi^2}$ captures covariance redundancy (second-order projection);  
in \emph{neuroscience and complex systems}, spectral redundancy $\mathcal{R}_{\mathrm{spec}}$ describes degeneracy and shared variance (eigen-structural projection).  
All are manifestations of the same underlying quantity that measures the departure of data from independence—%
the structure that sustains meaning in finite systems \cite{zollikofer2024beyond}.

Within this unified geometry, redundancy is shown to be bounded both above and below, giving rise to an intrinsic equilibrium $R^{*}$ between over-compression (loss of structure) and over-coupling (collapse).  
This equilibrium complements Shannon’s asymptotic ideal:  
while minimizing redundancy maximizes channel efficiency, maintaining an optimal level of redundancy enhances stability and generalization in structured, finite regimes.

More broadly, the framework provides a theoretical bridge between generative and discriminative paradigms of self-supervised learning \cite{nanda2023diffused}.  
In generative pretraining, such as masked autoencoders (MAE) or diffusion-based foundation models \cite{ho2020denoising, rombach2022high, he2024lotus, croitoru2023diffusion}, understanding emerges from the organization of data redundancy; generalization improves when redundancy is balanced rather than minimized.  
In contrast, discriminative or contrastive objectives, such as SimCLR or MoCo~\cite{chen2020simple, he2020momentum}, still tend to minimize redundancy, emphasizing separability over structure.  
This dichotomy highlights redundancy as a controllable variable that governs how information is structured, transferred, and understood across learning paradigms.

To illustrate this principle, we examine the generative regime through masked autoencoders (MAE).  
The experiments show that model understanding and generalization both peak when latent redundancy fluctuates around the theoretical equilibrium $R^{*}$.  
These experiments are not intended for performance benchmarking but to \emph{visualize and verify} the theoretical prediction that generative self-supervised learning achieves its most coherent internal organization at an optimal redundancy balance.  
We thus establish redundancy as a \emph{structural information principle}—a quantitative law unifying its projections across fields and revealing its equilibrium as the organizing center of modern representation learning.

The remainder of this paper is organized as follows.  
Section~\ref{sec:redundancy-framework} introduces the unified redundancy framework, defining $\mathcal{R}_f$ as a general $f$-divergence from independence and establishing its fundamental properties, including nonnegativity, the data-processing inequality, and Gaussian or spectral reductions.  
Section~\ref{sec:balance} develops the redundancy–loss equilibrium theory, proving the existence of an interior optimum $R^{*}$ and analyzing its stability under stochastic learning dynamics.  
Section~\ref{sec:sandwich} presents the information–theoretic ``sandwich'' argument, demonstrating that both insufficient and excessive redundancy impair performance, whereas a natural balance emerges at $R^{*}$.  
Section~\ref{sec:recoveries} provides empirical validation in masked autoencoders (MAE), confirming that model performance peaks near the theoretical optimum, and Appendix~\ref{app:proofs} contains complete mathematical proofs.  
Together, these sections establish a unified theory of redundancy that bridges information theory, statistical learning, and complex systems—elevating redundancy from a byproduct of compression to a governing principle of organization.

\begin{table}[h]
\centering
\caption{A cross-disciplinary summary of how redundancy has been defined, interpreted, and utilized across different scientific domains.  
Our framework unifies these views by treating redundancy as a measurable and optimizable property of data organization.}
\label{tab:redundancy_fields}
\vspace{0.4em}
\renewcommand{\arraystretch}{1.1}
\setlength{\tabcolsep}{5pt}
\begin{tabularx}{\linewidth}{lXXX}
\toprule
\textbf{Field / Domain} &
\textbf{Classical view of redundancy} &
\textbf{Recent perspective} &
\textbf{Representative works} \\
\midrule
\textbf{Information theory} &
Redundancy reduces channel capacity and should be minimized for efficient coding. &
Redundancy ensures reliability, error correction, and structured transmission in noisy environments. &
Shannon (1948); Huffman (1952); Ziv \& Lempel (1977). \\
\textbf{Neuroscience / Sensory coding} &
Redundant firing viewed as inefficiency in neural encoding. &
Overlapping receptive fields increase robustness, fault tolerance, and predictive stability. &
Barlow (2001); Narayanan (2005); Kazemivash et al. (2025). \\
\textbf{Statistics / Machine learning} &
Multicollinearity and overparameterization seen as overfitting or inefficiency. &
Redundant features improve generalization and stability; redundancy becomes a tunable regularization target. &
Zhang et al. (2016); Belkin et al. (2019); Li et al. (2023). \\
\textbf{Physics / Complex systems} &
Redundancy implies degeneracy and non-essential degrees of freedom. &
Redundancy enables stability, emergence, and robustness in high-dimensional systems. &
May (1972); Prigogine (1982); Lloyd (2007). \\
\textbf{Deep representation learning} &
Redundancy reduction as compression objective. &
Networks self-organize redundant subspaces that enhance generalization and robustness. &
Doimo et al. (2022); Nanda et al. (2023); Wollstadt et al. (2023). \\
\bottomrule
\end{tabularx}
\end{table}

\section{Related Works}

Redundancy has re-emerged as a central theme in recent studies of representation learning and overparameterized networks.  
Empirical analyses have shown that deep models often develop \emph{structured repetition} rather than mere inefficiency.  
Doimo \emph{et al.}~\cite{doimo2022redundant} demonstrated that wide neural networks naturally form groups of nearly identical neurons whose activations differ only by small perturbations, offering a mechanistic account of ``benign overfitting.''  
Similarly, Nanda \emph{et al.}~\cite{nanda2023diffused} reported \emph{diffuse redundancy} in pretrained representations: randomly subsampling neurons from intermediate layers hardly degrades downstream accuracy, indicating that information is distributed across many interchangeable units.  
Parallel work in information-theoretic feature analysis, such as Wollstadt \emph{et al.}~\cite{wollstadt2023rigorous}, reached similar conclusions through the lens of partial information decomposition (PID), formally partitioning predictive information into unique, synergistic, and redundant components.  
Together, these studies suggest that redundancy is not a statistical artifact to be removed, but an intrinsic property of high-dimensional representations that contributes to robustness and generalization.

Building on these empirical insights, subsequent research has sought to \emph{measure}, \emph{control}, and \emph{organize} redundancy in a principled way.  
The \textbf{RedTest} framework~\cite{lu2024redtest} provides a large-scale empirical methodology for quantifying neuron- and layer-level redundancy, revealing its correlation with robustness and energy efficiency.  
Yavuz and Yanikoglu~\cite{yavuz2025evaluating} proposed a coupling-matrix-based redundancy index $\rho(C)$ that captures latent-space correlations as a proxy for internal representational geometry, reinforcing the view that redundancy is a continuous structural variable rather than a defect.  
Wang \emph{et al.}~\cite{wang2020improving} interpreted redundancy reduction as a regularization objective in unsupervised domain adaptation, showing that selectively removing domain-specific redundancies improves transferability.  
More dynamically, Dorovatas \emph{et al.}~\cite{dorovatas2025auto} introduced \textbf{Auto-Compressing Networks} that self-organize to prune redundant pathways during training while maintaining representational quality, implicitly converging toward an internal redundancy equilibrium.  
Beyond deep learning, Li \emph{et al.}~\cite{li2023exploiting} demonstrated that exploiting intrinsic redundancy in large materials-science datasets enhances retrieval efficiency without compromising performance, underscoring the universality of redundancy as a reusable structural resource.

From a theoretical neuroscience perspective, Sajid \emph{et al.}~\cite{sajid2020degeneracy} provided a formal treatment of \emph{degeneracy} and \emph{redundancy} within the free-energy and active-inference framework.  
They defined degeneracy as the entropy of posterior beliefs—the flexibility afforded by multiple internal explanations for sensory outcomes—and redundancy as the complexity cost of forming those beliefs.  
Their analysis showed that adaptive agents minimize redundancy while maintaining sufficient degeneracy, achieving robust yet efficient inference.  
This formulation extends the role of redundancy beyond coding efficiency to a quantitative component of self-organization in both biological and synthetic intelligence, resonating with our view that redundancy regulates informational stability and adaptability rather than constituting noise.

Collectively, these developments trace a conceptual transition: from viewing redundancy as inefficiency to recognizing it as a \emph{principle of organization}.  
Our work builds on this transition by offering a unified mathematical framework that formalizes redundancy across domains and demonstrates its theoretical and empirical role in stabilizing learning and enhancing generalization.

\section{A Unified Redundancy Framework}
\label{sec:redundancy-framework}

\paragraph{Notation.}
Let $(\Omega,\mathcal{F})$ be a measurable space. 
A random vector $X=(X_1,\dots,X_n)$ takes values in the product space 
$\mathcal{X}=\mathcal{X}_1\times\cdots\times\mathcal{X}_n$.
For each coordinate, let $\mu_i$ be a $\sigma$-finite base measure on $\mathcal{X}_i$ and define 
$\mu=\bigotimes_{i=1}^n\mu_i$.
Assume that $P_X\ll\mu$ with joint density $p(x)=\tfrac{dP_X}{d\mu}(x)$, and that each marginal $P_{X_i}\ll\mu_i$ has density $p_i(x_i)=\tfrac{dP_{X_i}}{d\mu_i}(x_i)$.
The corresponding independent product measure is 
\[
\Pi_X \;=\; \bigotimes_{i=1}^n P_{X_i}, \qquad 
\frac{d\Pi_X}{d\mu}(x) = \prod_{i=1}^n p_i(x_i).
\]
For a positive semidefinite matrix $M$, denote by $\lambda(M)$ its eigenvalues, 
$\mathrm{tr}$ its trace, and $\|\cdot\|_F$ the Frobenius norm.
All logarithms are natural.

\subsection*{Master Definition: Redundancy as Divergence from Independence}

\begin{definition}[Unified redundancy functional]
\label{def:Rphi}
Let $f:\mathbb{R}_{>0}\to\mathbb{R}$ be a convex function satisfying $f(1)=0$.
The \emph{redundancy} of a random vector $X=(X_1,\dots,X_n)$ is defined as the 
$f$-divergence between its joint law and the independence manifold:
\begin{equation}
\label{eq:Rf}
\mathcal{R}_{f}(X)
\;\coloneqq\;
D_{f}\!\big(P_X \,\|\, \Pi_X\big)
\;=\;
\int_{\mathcal{X}} 
f\!\left(
\frac{p(x)}{\prod_{i=1}^n p_i(x_i)}
\right)
\prod_{i=1}^n p_i(x_i)\, d\mu(x)
\;=\;
\mathbb{E}_{\Pi_X}\!\left[
\,f\!\left(\frac{p(x)}{\prod_{i=1}^n p_i(x_i)}\right)
\right].
\end{equation}
This quantity is finite whenever $P_X\ll\Pi_X$; otherwise we set 
$\mathcal{R}_{f}(X)=+\infty$.
We refer to $\mathcal{R}_{f}$ as the \emph{redundancy functional} and 
to $f$ as its \emph{redundancy kernel}, specifying the geometry of the divergence.
\end{definition}

\begin{remark}[Specializations and the role of the kernel $f$]
\label{rem:kernel-role}
\textbf{Role of the kernel.}
In an $f$-divergence, the convex kernel $f$ determines how deviations from independence are measured and weighted.
\begin{enumerate}[label=(\roman*),leftmargin=*]
\item It defines the \emph{local geometry} near independence through $f''(1)$, which governs the quadratic (Gaussian) approximation of the divergence.
\item It controls the \emph{tail sensitivity} via the growth of $f(t)$ as $t\to 0$ or $t\to\infty$, specifying how strongly rare but large dependencies are penalized.
\item Its derivative $f'(t)$ acts as an \emph{influence function} (contrast function), determining which ranges of the likelihood ratio contribute most to the redundancy.
\item Different choices of $f$ induce distinct yet data-processing–consistent geometries on the independence manifold, all obeying the same monotonicity under coarse-graining.
\end{enumerate}
\end{remark}

\begin{itemize}
\item \textbf{Kullback–Leibler kernel:}  
For $f(t)=t\log t$,  
\[
\mathcal{R}_f(X)=D_{\mathrm{KL}}\!\left(P_X\middle\|\Pi_X\right),
\]
the total correlation (multi-information).  
This kernel grows superlinearly and is highly sensitive to rare but strong dependencies.

\item \textbf{Pearson $\chi^2$ kernel:}  
For $f(t)=\tfrac12(t-1)^2$,  
\[
\mathcal{R}_f(X)
=\tfrac12\!\int\!\frac{(p-\prod_i p_i)^2}{\prod_i p_i}\,dx,
\]
a quadratic-energy formulation that is less sensitive to tails and well suited to weak or near-Gaussian dependence.

\item \textbf{R\'enyi-$\alpha$ family:}  
Power-type kernels yield R\'enyi-style redundancies, tuning tail emphasis continuously with $\alpha$:  
larger $\alpha$ magnifies the influence of strong dependencies, while $\alpha\!\downarrow\!1$ recovers the KL case.

\item \textbf{Other classical choices:}  
Hellinger, Jensen–Shannon, and related kernels provide alternative redundancy geometries, each with its own balance between robustness and sensitivity.
\end{itemize}

\textbf{Summary.}
Information-theoretic, statistical, and geometric notions of redundancy are thus unified under the single principle
\[
\mathcal{R}_f(X)=D_f(P_X\|\Pi_X),
\]
where the kernel $f$ acts as the lens that shapes the sensitivity, robustness, and intrinsic geometry of the measure.

\subsection{Fundamental Properties of Redundancy}
\label{sec:fundamental}

\begin{proposition}[Nonnegativity and vanishing iff independence]
\label{prop:nonneg}
For any convex function $f:\mathbb{R}_{>0}\to\mathbb{R}$ satisfying $f(1)=0$, the redundancy functional satisfies
\[
\mathcal{R}_f(X) \ge 0,
\qquad
\text{and}\qquad
\mathcal{R}_f(X)=0
\;\Longleftrightarrow\;
P_X=\Pi_X
\quad (\text{i.e., } X_1,\dots,X_n \text{ are mutually independent}).
\]
\end{proposition}

\begin{proof}[Sketch]
This follows from standard properties of $f$-divergences:
for any probability measures $P,Q$ with $P\!\ll\!Q$, we have $D_f(P\|Q)\ge 0$, with equality if and only if $P=Q$ almost everywhere.
\end{proof}

\begin{proposition}[Data processing inequality (DPI)]
\label{prop:DPI}
Let $Y_i = g_i(X_i)$ be measurable mappings or coordinate-wise random channels, and let $Y=(Y_1,\dots,Y_n)$. Then
\[
\mathcal{R}_f(Y)
\;=\;
D_f\!\big(P_Y \,\|\, \Pi_Y\big)
\;\le\;
D_f\!\big(P_X \,\|\, \Pi_X\big)
\;=\;
\mathcal{R}_f(X).
\]
\end{proposition}

\begin{proof}[Sketch]
$f$-divergences are contractive under Markov kernels.  
In this setting, the overall channel acting on $X$ is the product of the coordinate-wise channels $(g_i)$, which maps $\Pi_X$ to $\Pi_Y$.  
Hence the contraction inequality implies $\mathcal{R}_f(Y)\le\mathcal{R}_f(X)$.
\end{proof}

\begin{proposition}[Bounds]
\label{prop:bounds}
Assume $P_X \ll \Pi_X$ and that $f(L(x))$ is finite on a $\Pi_X$-full set, where $L(x)=p(x)/\prod_i p_i(x_i)$. Then:
\begin{enumerate}
\item (Lower bound) $\mathcal{R}_f(X)\ge 0$ (Prop.~\ref{prop:nonneg}).
\item (Upper bounds) If $0<m\le L(x)\le M<\infty$ $\Pi_X$-a.e., then $\mathcal{R}_f(X)\le f(M)$; in particular for $f(t)=t\log t$, $\mathcal{R}_f(X)\le M\log M$.
\end{enumerate}
\end{proposition}

\subsection{Gaussian and Geometric Reductions}
\label{sec:gaussian-geometric}

\begin{proposition}[Gaussian total correlation]
\label{prop:gaussian_tc}
Let $X \sim \mathcal{N}(0,\Sigma)$ with correlation matrix $C=\mathrm{corr}(X)$.  
Then the redundancy under the Kullback–Leibler kernel equals the Gaussian total correlation:
\[
\mathcal{R}_{\mathrm{KL}}(X)
\;=\;
D_{\mathrm{KL}}\!\left(P_X\middle\|\Pi_X\right)
\;=\;
-\tfrac12\,\log\det C.
\]
\end{proposition}

\begin{proof}[Sketch]
For a centered Gaussian $X$, the joint density is
$p(x)\propto \exp\!\left(-\tfrac12 x^\top \Sigma^{-1}x\right)$,
and the independent-product density corresponds to replacing $\Sigma$ by its diagonal
$\mathrm{diag}(\Sigma)$.
Direct integration yields
\[
D_{\mathrm{KL}}\!\left(P_X\middle\|\Pi_X\right)
= \tfrac12\!\left(\mathrm{tr}(C)-\log\det C - n\right),
\]
which reduces to $-\tfrac12\log\det C$ when $C$ is a correlation matrix ($\mathrm{tr}(C)=n$).
\end{proof}

\begin{proposition}[Quadratic approximation and covariance form]
\label{prop:quad}
Let $C=\mathrm{corr}(X)$ and assume weak dependence (all off-diagonal entries of $C$ are small).  
For the quadratic kernel $f(t)=\tfrac12(t-1)^2$,
\[
\mathcal{R}_f(X)
\;\approx\;
\tfrac14\,\|C-I\|_F^2
\;=\;
\tfrac14\,\sum_{i\ne j} C_{ij}^2.
\]
Moreover, for $\|C-I\|_2$ sufficiently small,
\[
-\log\det C
\;\approx\;
\mathrm{tr}(I-C)
\;+\;
\tfrac12\,\|C-I\|_F^2,
\]
showing that the Gaussian total correlation $\mathcal{R}_{\mathrm{KL}}$ 
and the quadratic redundancy $\mathcal{R}_f$ coincide up to second order.
\end{proposition}

\begin{proof}[Sketch]
Second-order Taylor expansion of $f$ and $\log\det$ around independence $C=I$.
\end{proof}

\subsection{Spectral and Geometric Redundancy}
\label{sec:spectral}

\begin{definition}[Spectral (geometric) redundancy]
\label{def:spectral}
Let $Z \in \mathbb{R}^D$ be a random vector with covariance matrix $\Sigma_Z$.  
Denote its eigenvalues by $\lambda_1,\dots,\lambda_D>0$, and define the normalized spectrum
\[
\tilde{\lambda}_i \;=\; 
\frac{\lambda_i}{\sum_{j=1}^D \lambda_j},
\qquad
\sum_{i=1}^D \tilde{\lambda}_i = 1.
\]
The \emph{spectral entropy} and the corresponding \emph{effective rank} are
\[
H_\lambda(Z)
\;=\;
- \sum_{i=1}^{D} \tilde{\lambda}_i \log \tilde{\lambda}_i,
\qquad
r_{\mathrm{eff}}(Z)
\;=\;
\exp\!\big(H_\lambda(Z)\big).
\]
We define the \emph{spectral redundancy} as
\[
\mathcal{R}_{\mathrm{spec}}(Z)
\;=\;
1 - \frac{r_{\mathrm{eff}}(Z)}{D},
\qquad
\mathcal{R}_{\mathrm{spec}}(Z) \in \Big[\,0,\;1-\tfrac{1}{D}\,\Big].
\]
\end{definition}

\begin{proposition}[Bounds and extremal cases]
\label{prop:spec-bounds}
The spectral redundancy $\mathcal{R}_{\mathrm{spec}}(Z)$ satisfies:
\[
\mathcal{R}_{\mathrm{spec}}(Z)=0 
\quad\Longleftrightarrow\quad
\tilde{\lambda}_1=\cdots=\tilde{\lambda}_D=\tfrac{1}{D},
\]
i.e., when the spectrum is uniform (energy evenly distributed across dimensions);  
and
\[
\mathcal{R}_{\mathrm{spec}}(Z)
= 1-\tfrac{1}{D}
\quad\Longleftrightarrow\quad
\tilde{\lambda}_1 = 1,\;
\tilde{\lambda}_{i>1}=0,
\]
i.e., when $\Sigma_Z$ is rank-one (complete collapse onto a single direction).
\end{proposition}

\begin{proof}[Sketch]
The result follows from the properties of Shannon entropy.  
$H_\lambda$ is maximized for the uniform distribution, yielding $r_{\mathrm{eff}}=D$, and minimized when one component carries all the mass, giving $r_{\mathrm{eff}}=1$.  
Substituting these values into $\mathcal{R}_{\mathrm{spec}}=1-r_{\mathrm{eff}}/D$ establishes the bounds.
\end{proof}

\subsection{Redundancy Balance and the Existence of an Interior Optimum}
\label{sec:balance}

Let $\mathcal{R}(Z)$ denote any redundancy measure (e.g., $\mathcal{R}_{\mathrm{KL}}$, $\mathcal{R}_{\chi^2}$, or $\mathcal{R}_{\mathrm{spec}}$) 
evaluated on latent representations $Z = f_\theta(X)$.

\begin{assumption}[Loss–redundancy trade-off]
\label{assump:tradeoff}
Assume the training objective decomposes as
\[
\mathcal{L}(\theta)
\;=\;
\mathbb{E}\!\left[\ell_{\mathrm{rec}}(X;\theta)\right]
\;+\;
\lambda\,\mathcal{R}\!\big(f_\theta(X)\big),
\qquad \lambda>0.
\]
Let $D(R)$ denote the minimal achievable reconstruction risk at a fixed redundancy level $R\in[\underline{R},\overline{R}]$.  
Assume that $D(R)$ is continuous, strictly decreasing on $[\underline{R},R_1)$, 
and strictly increasing on $(R_2,\overline{R}]$ for some 
$\underline{R}<R_1\le R_2<\overline{R}$.
\end{assumption}

\begin{theorem}[Existence of a redundancy equilibrium]
\label{thm:equilibrium}
Under Assumption~\ref{assump:tradeoff}, 
the scalar objective
\[
\Phi(R)
\;=\;
D(R) + \lambda R
\]
admits at least one minimizer $R^{*}\in(\underline{R},\overline{R})$.  
If $D$ is twice differentiable and $D''(R^{*})>0$, 
then $R^{*}$ is a locally stable equilibrium point corresponding to the \emph{optimal redundancy}.
\end{theorem}

\begin{proof}
By continuity of $D(R)$ and the assumed monotonicity near the boundaries,
the derivative $\Phi'(R)=D'(R)+\lambda$ is negative for $R$ close to $\underline{R}$ 
and positive for $R$ close to $\overline{R}$.  
Hence, by the intermediate value theorem, there exists 
$R^{*}\in(\underline{R},\overline{R})$ satisfying $\Phi'(R^{*})=0$.  
If, in addition, $D''(R^{*})>0$, then $\Phi''(R^{*})>0$, 
and $R^{*}$ is a local minimizer—interpreted as a stable redundancy equilibrium.
\end{proof}

\begin{corollary}[Banded convergence signature]
\label{cor:band}
If the training dynamics $\theta_t$ monotonically decrease $\mathcal{L}(\theta_t)$ 
and both $\mathcal{L}(\theta_t)$ and $\mathcal{R}(f_{\theta_t}(X))$ 
are Lipschitz-continuous in $t$ with additive stochastic noise, 
then $\mathcal{R}(f_{\theta_t}(X))$ converges to a narrow band around $R^{*}$ 
and exhibits small fluctuations whose amplitude is proportional to the noise scale.
\end{corollary}

\subsection{Recovering Classical Notions and Practical Proxies}
\label{sec:recoveries}

\paragraph{Total correlation (multi-information).}
For $f(t)=t\log t$,
\[
\mathcal{R}_{\mathrm{KL}}(X)
\;=\;
\sum_{i=1}^n H(X_i) - H(X),
\]
which exactly measures the amount of \emph{shared information} among the coordinates.

\paragraph{Quadratic / covariance proxy.}
For near-Gaussian representations,
\[
\mathcal{R}_{\chi^2}(X)
\;\approx\;
\tfrac14\,\|C-I\|_F^2,
\]
providing a simple, differentiable redundancy regularizer that penalizes excess pairwise correlation.

\paragraph{Spectral proxy.}
The spectral redundancy
$\mathcal{R}_{\mathrm{spec}}(Z)$
is norm- and scale-invariant, computationally inexpensive, 
and captures the trade-off between \emph{collapse} (energy concentration) and 
\emph{fragmentation} (over-dispersion) in the representation spectrum.

\paragraph{Attention-head redundancy.}
Let $A_h \in \mathbb{R}^{N\times N}$ denote the attention map of head $h$.
A simple head-redundancy proxy is
\[
\mathcal{R}_{\mathrm{head}}
\;=\;
\frac{1}{H(H-1)}\!
\sum_{h\neq h'} 
\frac{\langle \mathrm{vec}(A_h),\,\mathrm{vec}(A_{h'}) \rangle^2}
{\|A_h\|_F^2\,\|A_{h'}\|_F^2},
\]
which lies in $[0,1]$ and contracts under head-wise post-processing, 
consistent with the data processing inequality for linear maps.

\subsection{Practical Lemmas for Boundedness During Training}
\label{sec:boundedness}

\begin{lemma}[Bound via spectral norm control]
\label{lem:spectral-bound}
Let $Z = W X$ where $\|W\|_2 \le M$ and $\mathbb{E}\|X\|^2 \le \sigma^2$.  
Then the covariance of $Z$ satisfies $\mathrm{tr}(\Sigma_Z) \le M^2 \sigma^2$, 
and consequently the effective rank $r_{\mathrm{eff}}(Z)$ lies in $[1,D]$.  
Hence,
\[
0 \;\le\; \mathcal{R}_{\mathrm{spec}}(Z)
\;\le\;
1 - \tfrac{1}{D}.
\]
\end{lemma}

\begin{proof}[Sketch]
Since $\Sigma_Z = W\,\Sigma_X\,W^{\top}$ and $\|W\|_2 \le M$,  
we have $\mathrm{tr}(\Sigma_Z) = \mathrm{tr}(W\Sigma_X W^{\top}) \le \|W\|_2^2\,\mathrm{tr}(\Sigma_X) \le M^2\sigma^2$.  
The normalization of eigenvalues then ensures $r_{\mathrm{eff}}(Z)\in[1,D]$, yielding the stated bounds.
\end{proof}

\begin{lemma}[KL–Frobenius bound near independence]
\label{lem:KL-Frob}
Let $C$ be a correlation matrix with spectral radius $\rho(C) < 1$, and define $A = C - I$.  
Then
\[
-\log\det C
\;=\;
-\mathrm{tr}\log(I + A)
\;\ge\;
\tfrac{1}{2}\|A\|_F^2 - \tfrac{1}{3}\|A\|_F^3,
\]
implying that for small $\|A\|_F$,
\[
\mathcal{R}_{\mathrm{KL}}(X)
\;\gtrsim\;
\tfrac{1}{2}\|C - I\|_F^2.
\]
\end{lemma}

\begin{proof}[Sketch]
Using the Taylor expansion $\log(I{+}A) = A - \tfrac12A^2 + \tfrac13A^3 - \cdots$ 
and $\mathrm{tr}(A)=0$ for correlation matrices,  
we obtain $-\log\det C = -\mathrm{tr}\log(I{+}A) \ge \tfrac12\|A\|_F^2 - \tfrac13\|A\|_F^3$.
\end{proof}

\subsection{Information-Theoretic Sandwich Proof: Why Redundancy Is Not ``The Less, the Better''}
\label{sec:sandwich}

We now demonstrate that the classical Shannon view of redundancy—interpreting it purely as inefficiency to be minimized—fails to hold in finite, noisy, and structured regimes.  
Instead, the task error $\mathcal{E}(R)$ (or risk as a function of redundancy $R$) typically follows a \emph{U-shaped} dependence:  
both insufficient and excessive redundancy degrade performance, 
while an intermediate value $R^{*}$ yields optimal generalization and manifests as a stable equilibrium band around which $\mathcal{R}(f_{\theta_t}(X))$ fluctuates during training.

\paragraph{Setting.}
Let $Z = f_\theta(X)$ be a learned representation of data $X$ that is relevant to a downstream target $S$
(e.g., a latent factor, class label, or generative source).  
Denote redundancy by $R = \mathcal{R}(Z)$, measured by any $f$-divergence form (Def.~\ref{def:Rphi}) or by its geometric proxies such as $\mathcal{R}_{\chi^2}$ or $\mathcal{R}_{\mathrm{spec}}$.
Assume a bounded entropy budget $\sum_i H(Z_i)\le C_0$ and that the channel $X\!\to\!Z$ is affected by finite noise or perturbation power.

\begin{lemma}[Capacity-side lower bound]
\label{lem:capacity}
The mutual information between $Z$ and the task variable $S$ satisfies
\[
I(Z;S)
\;\le\;
H(Z)
\;\le\;
\sum_i H(Z_i) - \mathrm{TC}(Z)
\;=\;
C_0 - \mathcal{R}_{\mathrm{KL}}(Z),
\]
and therefore
\[
\mathcal{E}(R)
\;\ge\;
g_{\mathrm{info}}(C_0 - R),
\]
where $g_{\mathrm{info}}$ is a strictly increasing function.
\end{lemma}

\begin{proof}[Sketch]
From the entropy decomposition
$H(Z) = \sum_i H(Z_i) - \mathrm{TC}(Z)$
and the finite entropy budget $\sum_i H(Z_i) \le C_0$,
we obtain
$H(Z) \le C_0 - \mathrm{TC}(Z)$.
Since $I(Z;S) \le H(Z)$,  
either Fano's inequality (for classification) or the I--MMSE relationship (for Gaussian regression)
implies that the task risk increases monotonically as $I(Z;S)$ decreases.  
Hence, excessive redundancy (large total correlation) reduces the effective information capacity available for the task.
\end{proof}

\begin{lemma}[Robustness-side upper bound]
\label{lem:robustness}
If the encoding $X\!\to\!Z$ and decoding $Z\!\to\!\hat S$ operate under bounded stochastic noise, dropout, or adversarial perturbations, then
\[
\mathcal{E}(R)
\;\le\;
g_{\mathrm{robust}}(R),
\]
for some continuous function $g_{\mathrm{robust}}$ that is strictly decreasing on $[0, R_1)$.
\end{lemma}

\begin{proof}[Sketch]
When redundancy is close to zero (i.e., the coordinates of $Z$ are nearly independent), 
the representation lacks alternative pathways for recovering corrupted or missing information.  
Information-theoretic rate–distortion bounds and classical error-correction arguments show that 
replication or correlated encoding across coordinates (nonzero redundancy) 
reduces the expected reconstruction or classification error.  
Consequently, increasing small amounts of redundancy strictly improves robustness to perturbations,
yielding a decreasing error function $g_{\mathrm{robust}}(R)$ in the low-redundancy regime.
\end{proof}

\subsection{U-shape and Interior Optimum}
\label{sec:Ushape}

\begin{theorem}[Interior optimal redundancy and stability band]
\label{thm:Ushape}
Assume that $\mathcal{E}(R)$ is continuous and satisfies
\[
g_{\mathrm{robust}}(R)
\;\le\;
\mathcal{E}(R)
\;\le\;
g_{\mathrm{info}}(C_0{-}R),
\]
where $g_{\mathrm{robust}}$ is strictly decreasing on $[0,R_1)$ and 
$g_{\mathrm{info}}$ is strictly increasing on $(R_2,C_0]$ 
for some $0<R_1\le R_2<C_0$.
Then there exists an interior minimizer $R^{*}\in(R_1,R_2)$ such that
\[
\mathcal{E}(R^{*})
\;=\;
\min_{R\in[0,C_0]} \mathcal{E}(R).
\]
If the training dynamics follow noisy gradient descent that decreases $\mathcal{E}$ 
and $\mathcal{R}(f_{\theta_t}(X))$ is Lipschitz-continuous in time, 
then $\mathcal{R}(f_{\theta_t}(X))$ converges to a stability band 
$[R_{\min}, R_{\max}]$ around $R^{*}$.
\end{theorem}

\begin{proof}
By Lemmas~\ref{lem:capacity}--\ref{lem:robustness}, $\mathcal{E}(R)$ is bounded between two functions 
with opposing monotonic trends near the boundaries.  
By continuity, $\mathcal{E}'(R)$ changes sign, and by the intermediate value theorem, 
there exists $R^{*}$ such that $\mathcal{E}'(R^{*})=0$.  
Local convexity of $g_{\mathrm{info}}$ and $g_{\mathrm{robust}}$ implies that 
$\mathcal{E}''(R^{*})>0$, yielding a stable minimum.  
Under small gradient noise, stochastic approximation theory ensures convergence 
to a narrow fluctuation band $[R_{\min}, R_{\max}]$ around this equilibrium point.
\end{proof}

\begin{corollary}[Operational interpretation]
For $R<R_1$, the system is \emph{under-redundant}—fragile and non-generalizing.  
For $R>R_2$, it is \emph{over-redundant}—collapsed and low in effective capacity.  
Between them lies a \emph{redundancy equilibrium band} 
$[R_{\min}, R_{\max}]$ where efficiency, robustness, and generalization coexist in balance.
\end{corollary}

\paragraph{Interpretation.}
Shannon’s 1948 paradigm addressed the asymptotic and noiseless limit, 
in which redundancy merely wastes bits.  
In contrast, real-world and learning systems operate in finite-sample, noisy, and structured regimes, 
where redundancy serves as a functional resource rather than a defect.  
It provides  
(i) structural stability through multiple consistent cues,  
(ii) error correction and smooth interpolation, and  
(iii) invariance under perturbations.  
Thus, \emph{redundancy is not the opposite of information, but the geometry that sustains it}.  
An optimal, nonzero level of redundancy forms the basis of structure itself—and constitutes the mathematical condition for understanding.

\section{Experiments}
\label{sec:experiments}

\paragraph{Goal.}
The goal of our experiments is to empirically validate the theoretical predictions of the redundancy–balance framework rather than to introduce any architectural innovation.  
Specifically, we test Theorem~\ref{thm:Ushape}, which predicts that the generalization performance of a learning system depends non-monotonically on its internal redundancy $R$: both insufficient and excessive redundancy impair stability, while an intermediate value $R^{*}$ yields optimal organization and generalization.

\paragraph{Setup.}
We adopt the masked autoencoder (MAE)~\cite{he2022masked} as a representative self-supervised architecture and incorporate our redundancy regularization into its pretraining objective.  
The model reconstructs masked image patches under a fixed mask ratio of 0.75 using a ViT-Base encoder–decoder backbone.  
To test the redundancy–balance hypothesis, we augment the original reconstruction loss with a spectral-redundancy term, yielding the total objective.
\begin{equation}
\label{eq:mae_loss}
\mathcal{L}_{\text{total}}
\;=\;
\mathcal{L}_{\text{recon}}
\;+\;
\lambda_{\mathrm{red}}\,
\mathcal{R}_{\mathrm{spec}}(Z),
\qquad
\text{where}\quad
\mathcal{L}_{\text{recon}}
=\|\,\hat{X}-X\,\|_2^2.
\end{equation}
Here, $\mathcal{R}_{\mathrm{spec}}(Z)=1-r_{\mathrm{eff}}(Z)/D$ is the spectral-redundancy functional defined in Section~\ref{sec:spectral}.  
The coefficient $\lambda_{\mathrm{red}}$ controls the trade-off between pixel-level reconstruction fidelity and structural redundancy in the latent representation $Z$.  
By varying $\lambda_{\mathrm{red}}\!\in\!\{0,10^{\!-\!3},10^{\!-\!2},5{\times}10^{2}\}$, we obtain models spanning distinct redundancy regimes.  
Except for this single coefficient, all experimental settings—including architecture, optimizer, learning-rate schedule, data augmentation, and mask ratio—are kept identical across runs.  
After pretraining, the encoder is frozen and evaluated through a linear-probe classifier on CIFAR-100, providing a downstream measure of generalization ability.

\paragraph{Metrics.}
We monitor two redundancy proxies during pretraining:  
(i) the covariance-based redundancy $\mathcal{R}_{\chi^2}$, and  
(ii) the spectral redundancy $\mathcal{R}_{\mathrm{spec}}$, computed from the effective rank $r_{\mathrm{eff}}$ of the covariance spectrum as $1-r_{\mathrm{eff}}/D$.  
While validation losses remain nearly identical across different $\lambda_{\mathrm{red}}$, the resulting representations exhibit markedly different redundancy profiles, making these metrics essential for quantifying structural organization.  
Downstream generalization is assessed via the linear-probe Top-1 accuracy on CIFAR-100, which serves as a sensitive indicator of information balance.

\paragraph{Training protocol.}
All models are trained for the same number of epochs with identical random seeds, optimizer hyperparameters, learning-rate schedules, and data augmentations.  
**This strict control ensures that any observed difference in generalization or redundancy originates solely from the redundancy-regularization term $\lambda_{\mathrm{red}}$.**  
Redundancy statistics are logged at each epoch to analyze both equilibrium levels and dynamic trajectories during training.

\paragraph{Summary.}
This experimental design isolates the effect of redundancy regulation in a fully controlled setting and directly tests the theoretical predictions of equilibrium and self-organization described in Sections~\ref{sec:balance} and~\ref{sec:Ushape}.

\section{Results and Analysis}
\label{sec:results}

\begin{figure}[t]
    \centering
    \includegraphics[width=0.72\linewidth]{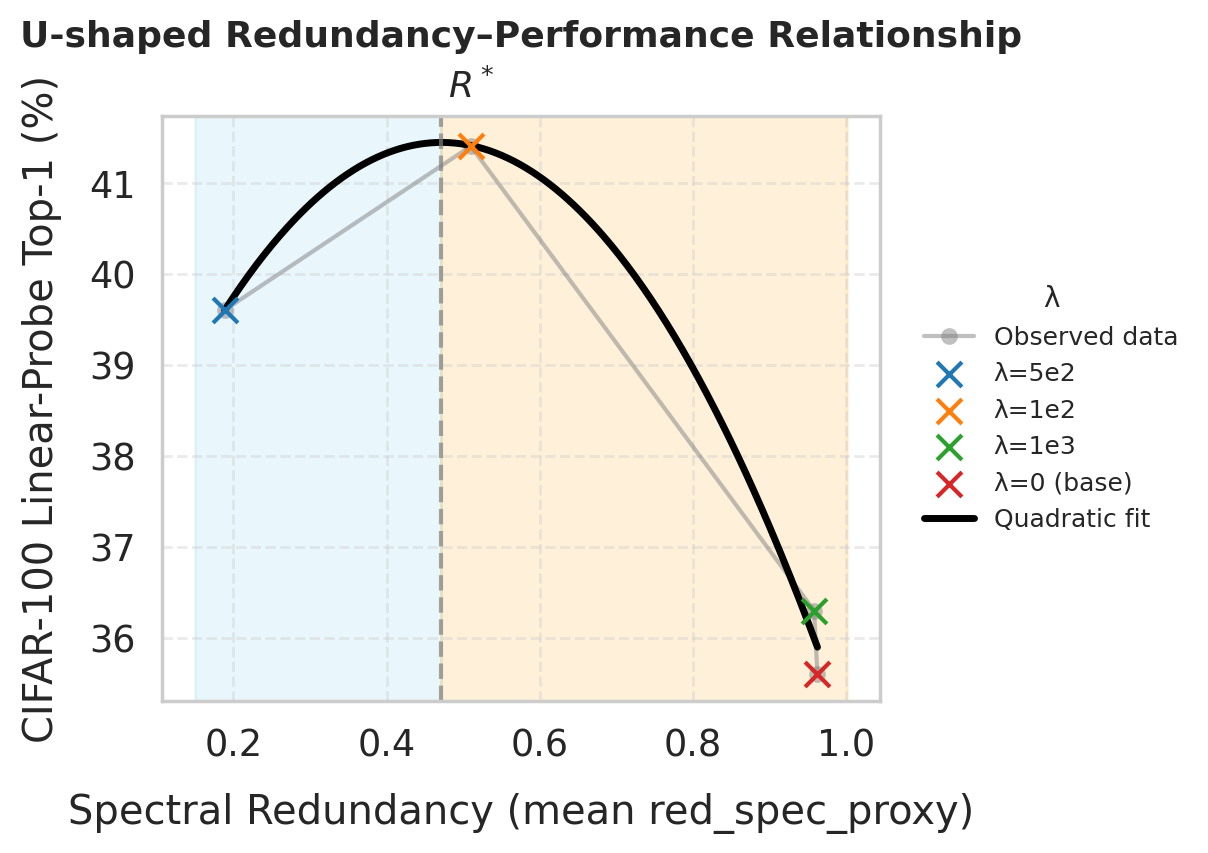}
    \vspace{-0.4em}
    \caption{
    \textbf{U-shaped redundancy–performance relationship.}
    CIFAR-100 linear-probe Top-1 accuracy (vertical axis) varies non-monotonically with the spectral redundancy 
    $\mathcal{R}_{\mathrm{spec}}$ (horizontal axis).  
    Both over-redundant (high $\mathcal{R}_{\mathrm{spec}}$) and under-redundant (low $\mathcal{R}_{\mathrm{spec}}$) models exhibit degraded performance, 
    while an intermediate redundancy ($\lambda_{\mathrm{red}}{=}10^{-2}$, $\mathcal{R}_{\mathrm{spec}}\!\approx\!0.5$) achieves the best generalization (41.4\%).  
    The dashed line marks the fitted optimum $R^{*}$ predicted by Theorem~\ref{thm:Ushape}, 
    and shaded regions indicate under- and over-redundant regimes.  
    }
    \label{fig:Ushape}
\end{figure}

\paragraph{Empirical validation of the redundancy–balance theorem.}
Figure~\ref{fig:Ushape} shows the empirical relationship between spectral redundancy and downstream generalization performance.  
The results confirm the key prediction of Theorem~\ref{thm:Ushape}: generalization follows a clear U-shaped dependence on redundancy.  
When no redundancy regularization is applied ($\lambda_{\mathrm{red}}{=}0$), the model exhibits high spectral coupling ($\mathcal{R}_{\mathrm{spec}}\!\approx\!0.96$) and poor generalization ($35.6\%$ Top-1 on CIFAR-100).  
Introducing mild redundancy control ($\lambda_{\mathrm{red}}{=}10^{-3}$) slightly improves generalization ($36.3\%$) but remains in the over-coupled regime.  
At an intermediate regularization strength ($\lambda_{\mathrm{red}}{=}10^{-2}$), spectral redundancy decreases to approximately $0.5$ and yields the best generalization ($41.4\%$).  
When regularization becomes too strong ($\lambda_{\mathrm{red}}{=}5{\times}10^2$), redundancy is over-suppressed ($\mathcal{R}_{\mathrm{spec}}\!\approx\!0.19$), leading to a decline in performance ($39.6\%$).  
This non-monotonic trend demonstrates the existence of an interior optimum $R^{*}$ balancing compression and coupling, consistent with the theoretical redundancy equilibrium.

\paragraph{Quantitative summary.}
Table~\ref{tab:mae_results} summarizes the relationship between redundancy, validation loss, and downstream generalization across different regularization regimes.  
While validation losses remain nearly identical across runs, the redundancy levels and probe accuracies vary significantly, indicating that the improvement in generalization arises from structural organization rather than overfitting or optimization bias.

\begin{table}[h]
\centering
\caption{Summary of MAE pretraining under varying redundancy regularization strengths.  
Spectral redundancy $\mathcal{R}_{\mathrm{spec}}$ is computed as the mean value over the final training epochs.  
Top-1 accuracy is measured by linear probing on CIFAR-100.}
\label{tab:mae_results}
\vspace{0.5em}
\renewcommand{\arraystretch}{1.1}
\setlength{\tabcolsep}{4pt}
\begin{tabularx}{\linewidth}{lXXXX}
\toprule
\textbf{Redundancy weight} $\boldsymbol{\lambda_{\mathrm{red}}}$ &
\textbf{Spectral redundancy} $\boldsymbol{\mathcal{R}_{\mathrm{spec}}}$ &
\textbf{Validation loss} &
\textbf{CIFAR-100 Top-1 (\%)} &
\textbf{Regime characterization} \\
\midrule
$0$ (baseline)     & $0.96$ & $0.0293$ & $35.6$ & Over-coupled (collapse) \\
$10^{-3}$          & $0.96$ & $0.0293$ & $36.3$ & Slightly over-coupled \\
$10^{-2}$          & $0.51$ & $0.0297$ & $\mathbf{41.4}$ & Balanced (near $R^{*}$) \\
$5{\times}10^{2}$  & $0.19$ & $0.0315$ & $39.6$ & Under-redundant (fragmented) \\
\bottomrule
\end{tabularx}
\end{table}

\paragraph{Interpretation.}
The observed U-shaped pattern provides direct empirical support for the redundancy–balance theory.  
It indicates that redundancy is not merely a form of noise to be suppressed, but a structured degree of freedom that must be maintained at an optimal level to achieve stability and generalization.  
This equilibrium behavior also implies that learning dynamics naturally converge to a self-organized regime around $R^{*}$, rather than requiring fine-tuned regularization.  
Together, these results establish redundancy as a measurable and predictive quantity governing the geometry of learned representations.  
This motivates further examination of how redundancy evolves during training and whether it naturally stabilizes around the predicted equilibrium $R^{*}$.

\paragraph{Training dynamics and equilibrium formation.}
Figure~\ref{fig:rank_dynamics} illustrates the joint evolution of effective rank (solid lines) and spectral redundancy (dashed lines) across different regularization strengths $\lambda_{\mathrm{red}}$.  
The trajectories reveal three key phenomena consistent with the redundancy–balance theory.  
First, redundancy regularization exerts a \emph{nonlinear} influence on representational structure: small increases in $\lambda_{\mathrm{red}}$ induce moderate reductions in redundancy, whereas large values (e.g., $5\times10^2$) lead to a rapid collapse of redundancy to near-zero levels.  
This sensitivity implies that $\lambda_{\mathrm{red}}$ controls a structural, rather than purely statistical, phase transition in the organization of representations.

Second, redundancy and effective rank evolve as strongly coupled variables, showing clear inverse trajectories.  
As redundancy decreases, the effective dimensionality of latent features expands, indicating that redundancy acts as an internal regulator of usable representational degrees of freedom.  
However, when redundancy is over-suppressed, the representation becomes overly fragmented—large in rank but poor in coherence—highlighting that high dimensionality alone does not guarantee generalization.

Third, moderate regularization ($\lambda_{\mathrm{red}}\!=\!10^{-2}$) leads to a stable \emph{equilibrium band}, where redundancy and effective rank co-stabilize after early fluctuations.  
This regime corresponds to the predicted redundancy equilibrium $R^{*}$: a self-organizing attractor in learning dynamics where compression and coupling are balanced.  
The emergence of this equilibrium without fine-tuning supports the theoretical claim that redundancy is a self-organizing property of data manifolds, not an artifact of optimization.

\begin{figure}[h]
    \centering
    \includegraphics[width=0.8\linewidth]{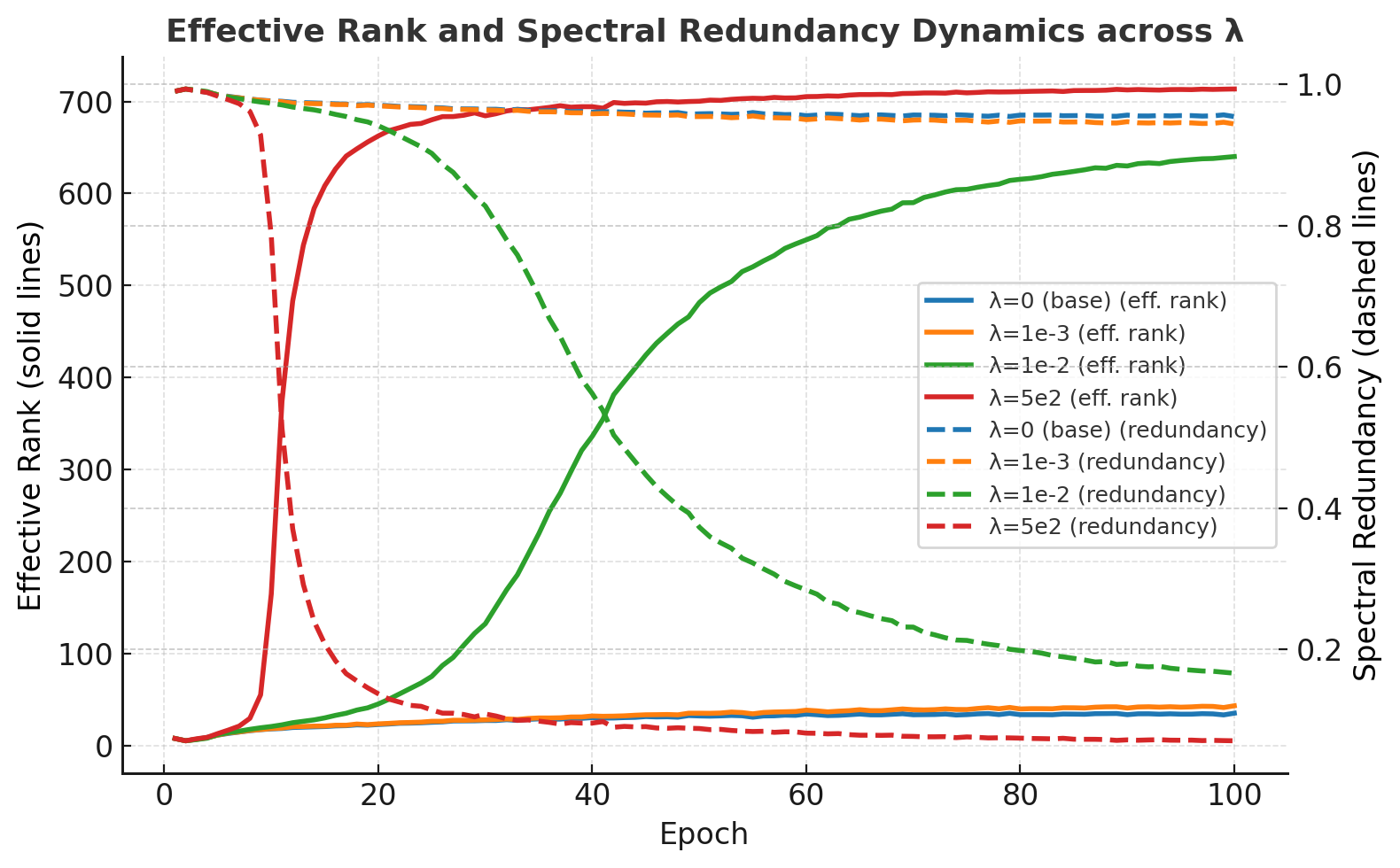}
    \vspace{-0.3em}
    \caption{
    \textbf{Training dynamics of effective rank and spectral redundancy across different redundancy regularization strengths $\lambda_{\mathrm{red}}$.}
    Solid lines denote the evolution of effective rank, and dashed lines denote spectral redundancy.  
    Moderate regularization ($\lambda_{\mathrm{red}}\!=\!10^{-2}$) yields a stable equilibrium where redundancy and rank co-stabilize, while excessive regularization ($5\times10^2$) collapses redundancy.  
    The relationship between $\lambda_{\mathrm{red}}$ and redundancy is markedly non-linear, confirming that redundancy regulation induces a self-organizing equilibrium rather than monotonic suppression.
    }
    \label{fig:rank_dynamics}
\end{figure}

\section{Discussion}
\label{sec:discussion}

\paragraph{Summary and perspective.}
This work proposes a theoretical paradigm that redefines redundancy as a fundamental quantity of information organization rather than inefficiency.  
We formulate a unified family of redundancy functionals $\{\mathcal{R}_f\}$ based on $f$-divergences from statistical independence, showing that classical measures such as mutual information, $\chi^2$ dependence, and spectral redundancy are projections of the same underlying geometry.  
The theory demonstrates that redundancy is bounded both above and below, giving rise to an intrinsic equilibrium $R^{*}$ between over-compression and over-coupling.  
Regularizing redundancy toward this equilibrium enhances stability and generalization in finite, structured regimes.  
Together, these results establish redundancy as a \emph{structural information principle}—a paradigm shift that extends Shannon’s efficiency-centric information theory into the finite, noisy, and organized world of learning systems.

\paragraph{Redundancy and foundation models.}
Our framework provides conceptual and practical insights for the design of \emph{generative self-supervised and foundation models}.  
These systems—ranging from masked autoencoders and diffusion models to large language and multimodal models \cite{ho2020denoising, rombach2022high, he2024lotus, caro2023brainlm, cox2024brainsegfounder, wu2025adabrain, sun2025foundation}—learn by reconstructing or predicting structured dependencies within data, a process inherently governed by redundancy.  
The redundancy–balance principle suggests that such generative models achieve their best understanding and transferability when latent redundancy is regulated near its equilibrium $R^{*}$, maintaining a stable internal geometry between over-coupling and over-compression.  
In practice, \emph{explicit redundancy regularization}—for instance via spectral or covariance-based constraints—could provide a self-organizing mechanism that stabilizes large generative pretraining and enhances robustness across modalities.  
Thus, our theory complements current scale-driven approaches by offering a structural principle for efficient, coherent, and generalizable generative learning.

\paragraph{Limitations and open directions.}
While the proposed theory provides a unifying view of redundancy, several limitations remain.  
First, our empirical validation focuses on generative self-supervised learning (e.g., masked autoencoders) rather than on large-scale foundation models such as GPTs \cite{achiam2023gpt} or multimodal LLMs \cite{xu2024survey}.  
Although we hypothesize that the redundancy–balance principle also governs their emergent generalization, confirming this requires extensive large-scale experiments.  
Second, the equilibrium theorem is not expected to be universally optimal: discriminative or contrastive paradigms may still benefit from minimizing redundancy to maximize separability.  
The framework thus captures one axis of the generative–discriminative continuum, but not the entirety of learning dynamics.  
Third, the theory has not yet been tested in cross-disciplinary domains such as neuroscience, where redundancy is hypothesized to underlie degeneracy and robustness in brain activity.  
Validating the redundancy equilibrium on neural and cognitive datasets will be an essential step toward bridging machine and biological learning.  
Finally, while our analysis concentrates on vector-valued data, redundancy in other modalities—particularly graphs, multimodal signals, and spatiotemporal structures—remains largely unexplored.  
Quantifying “graph redundancy” and integrating it with graph neural networks presents especially promising directions for future work \cite{gates2021effective, kumar2024spectral}.

\paragraph{Outlook.}
Despite these limitations, the redundancy–balance principle offers a new foundation for understanding learning systems as self-organizing processes that regulate—not eliminate—dependence.  
It reframes generalization as a property of data organization rather than model complexity, suggesting that redundancy forms the structural backbone through which stability and understanding emerge.  
Future research should explore how data redundancy interacts with other fundamental factors such as model capacity, architectural complexity, and optimization dynamics.  
Jointly characterizing these variables may yield a unified theory of learning equilibria, where data structure and model structure co-adapt to maintain informational balance.  
Such an extension could guide the principled design of large-scale foundation models, multimodal architectures, and biological neural systems alike.  
In this broader view, redundancy may ultimately serve as the missing link connecting statistical learning, information theory, and cognitive computation—a structural law governing how complex systems learn to understand.

\bibliographystyle{unsrtnat}
\bibliography{references}  

\begin{thebibliography}{45}
\providecommand{\natexlab}[1]{#1}
\providecommand{\url}[1]{\texttt{#1}}
\expandafter\ifx\csname urlstyle\endcsname\relax
  \providecommand{\doi}[1]{doi: #1}\else
  \providecommand{\doi}{doi: \begingroup \urlstyle{rm}\Url}\fi

\bibitem[Shannon(1948)]{shannon1948mathematical}
Claude~E Shannon.
\newblock A mathematical theory of communication.
\newblock \emph{The Bell system technical journal}, 27\penalty0 (3):\penalty0 379--423, 1948.

\bibitem[Barlow(2001)]{barlow2001redundancy}
Horace Barlow.
\newblock Redundancy reductionrevisited.
\newblock \emph{Network: computation in neural systems}, 12\penalty0 (3):\penalty0 241, 2001.

\bibitem[Narayanan et~al.(2005)Narayanan, Kimchi, and Laubach]{narayanan2005redundancy}
Nandakumar~S Narayanan, Eyal~Y Kimchi, and Mark Laubach.
\newblock Redundancy and synergy of neuronal ensembles in motor cortex.
\newblock \emph{Journal of Neuroscience}, 25\penalty0 (17):\penalty0 4207--4216, 2005.

\bibitem[Tononi et~al.(1994)Tononi, Sporns, and Edelman]{tononi1994measure}
Giulio Tononi, Olaf Sporns, and Gerald~M Edelman.
\newblock A measure for brain complexity: relating functional segregation and integration in the nervous system.
\newblock \emph{Proceedings of the National Academy of Sciences}, 91\penalty0 (11):\penalty0 5033--5037, 1994.

\bibitem[Kazemivash et~al.(2025)Kazemivash, Suresh, Ye, Iraji, Liu, Plis, Kochunov, Zhu, and Calhoun]{kazemivash2025st}
Behnam Kazemivash, Pranav Suresh, Dong~Hye Ye, Armin Iraji, Jingyu Liu, Sergey Plis, Peter Kochunov, David~C Zhu, and Vince~D Calhoun.
\newblock st-densevit: A weakly supervised spatiotemporal vision transformer for dense prediction of dynamic brain networks.
\newblock \emph{Human Brain Mapping}, 46\penalty0 (14):\penalty0 e70364, 2025.

\bibitem[Gunst and Webster(1975)]{gunst1975regression}
RF~Gunst and JT~Webster.
\newblock Regression analysis and problems of multicollinearity.
\newblock \emph{Communications in Statistics-Theory and Methods}, 4\penalty0 (3):\penalty0 277--292, 1975.

\bibitem[Zhang et~al.(2016)Zhang, Bengio, Hardt, Recht, and Vinyals]{zhang2016understanding}
Chiyuan Zhang, Samy Bengio, Moritz Hardt, Benjamin Recht, and Oriol Vinyals.
\newblock Understanding deep learning requires rethinking generalization.
\newblock \emph{arXiv preprint arXiv:1611.03530}, 2016.

\bibitem[Belkin et~al.(2019)Belkin, Hsu, Ma, and Mandal]{belkin2019reconciling}
Mikhail Belkin, Daniel Hsu, Siyuan Ma, and Soumik Mandal.
\newblock Reconciling modern machine-learning practice and the classical bias--variance trade-off.
\newblock \emph{Proceedings of the National Academy of Sciences}, 116\penalty0 (32):\penalty0 15849--15854, 2019.

\bibitem[Li et~al.(2023)Li, Persaud, Choudhary, DeCost, Greenwood, and Hattrick-Simpers]{li2023exploiting}
Kangming Li, Daniel Persaud, Kamal Choudhary, Brian DeCost, Michael Greenwood, and Jason Hattrick-Simpers.
\newblock Exploiting redundancy in large materials datasets for efficient machine learning with less data.
\newblock \emph{Nature Communications}, 14\penalty0 (1):\penalty0 7283, 2023.

\bibitem[Berchenko(2024)]{berchenko2024simplicity}
Yakir Berchenko.
\newblock Simplicity bias in overparameterized machine learning.
\newblock In \emph{Proceedings of the AAAI Conference on Artificial Intelligence}, volume~38, pages 11052--11060, 2024.

\bibitem[Allen-Zhu et~al.(2019)Allen-Zhu, Li, and Song]{allen2019convergence}
Zeyuan Allen-Zhu, Yuanzhi Li, and Zhao Song.
\newblock A convergence theory for deep learning via over-parameterization.
\newblock In \emph{International conference on machine learning}, pages 242--252. PMLR, 2019.

\bibitem[Chan et~al.(2022)Chan, Leow, Bea, Cheng, Phoong, Hong, and Chen]{chan2022mitigating}
Jireh Yi-Le Chan, Steven Mun~Hong Leow, Khean~Thye Bea, Wai~Khuen Cheng, Seuk~Wai Phoong, Zeng-Wei Hong, and Yen-Lin Chen.
\newblock Mitigating the multicollinearity problem and its machine learning approach: a review.
\newblock \emph{Mathematics}, 10\penalty0 (8):\penalty0 1283, 2022.

\bibitem[May(1972)]{may1972will}
Robert~M May.
\newblock Will a large complex system be stable?
\newblock \emph{Nature}, 238\penalty0 (5364):\penalty0 413--414, 1972.

\bibitem[Lloyd(2007)]{lloyd2007programming}
Seth Lloyd.
\newblock \emph{Programming the universe: a quantum computer scientist takes on the cosmos}.
\newblock Vintage, 2007.

\bibitem[Prigogine and Hiebert(1982)]{prigogine1982being}
Ilya Prigogine and Erwin~N Hiebert.
\newblock From being to becoming: Time and complexity in the physical sciences, 1982.

\bibitem[Huffman(2007)]{huffman2007method}
David~A Huffman.
\newblock A method for the construction of minimum-redundancy codes.
\newblock \emph{Proceedings of the IRE}, 40\penalty0 (9):\penalty0 1098--1101, 2007.

\bibitem[Ziv and Lempel(2003)]{ziv2003universal}
Jacob Ziv and Abraham Lempel.
\newblock A universal algorithm for sequential data compression.
\newblock \emph{IEEE Transactions on information theory}, 23\penalty0 (3):\penalty0 337--343, 2003.

\bibitem[Hyv{\"a}rinen and Oja(2000)]{hyvarinen2000independent}
Aapo Hyv{\"a}rinen and Erkki Oja.
\newblock Independent component analysis: algorithms and applications.
\newblock \emph{Neural networks}, 13\penalty0 (4-5):\penalty0 411--430, 2000.

\bibitem[Bell and Sejnowski(1997)]{bell1997independent}
Anthony~J Bell and Terrence~J Sejnowski.
\newblock The “independent components” of natural scenes are edge filters.
\newblock \emph{Vision research}, 37\penalty0 (23):\penalty0 3327--3338, 1997.

\bibitem[Tishby et~al.(2000)Tishby, Pereira, and Bialek]{tishby2000information}
Naftali Tishby, Fernando~C Pereira, and William Bialek.
\newblock The information bottleneck method.
\newblock \emph{arXiv preprint physics/0004057}, 2000.

\bibitem[Olshausen and Field(1996)]{olshausen1996emergence}
Bruno~A Olshausen and David~J Field.
\newblock Emergence of simple-cell receptive field properties by learning a sparse code for natural images.
\newblock \emph{Nature}, 381\penalty0 (6583):\penalty0 607--609, 1996.

\bibitem[Zollikofer et~al.(2024)Zollikofer, Egressy, Benzing, Otth, and Wattenhofer]{zollikofer2024beyond}
David Zollikofer, B{\'e}ni Egressy, Frederik Benzing, Matthias Otth, and Roger Wattenhofer.
\newblock Beyond pairwise correlations: Higher-order redundancies in self-supervised representation learning.
\newblock \emph{arXiv preprint arXiv:2412.01926}, 2024.

\bibitem[Nanda et~al.(2023)Nanda, Speicher, Dickerson, Gummadi, Feizi, and Weller]{nanda2023diffused}
Vedant Nanda, Till Speicher, John Dickerson, Krishna Gummadi, Soheil Feizi, and Adrian Weller.
\newblock Diffused redundancy in pre-trained representations.
\newblock \emph{Advances in Neural Information Processing Systems}, 36:\penalty0 4055--4079, 2023.

\bibitem[Ho et~al.(2020)Ho, Jain, and Abbeel]{ho2020denoising}
Jonathan Ho, Ajay Jain, and Pieter Abbeel.
\newblock Denoising diffusion probabilistic models.
\newblock \emph{Advances in neural information processing systems}, 33:\penalty0 6840--6851, 2020.

\bibitem[Rombach et~al.(2022)Rombach, Blattmann, Lorenz, Esser, and Ommer]{rombach2022high}
Robin Rombach, Andreas Blattmann, Dominik Lorenz, Patrick Esser, and Bj{\"o}rn Ommer.
\newblock High-resolution image synthesis with latent diffusion models.
\newblock In \emph{Proceedings of the IEEE/CVF conference on computer vision and pattern recognition}, pages 10684--10695, 2022.

\bibitem[He et~al.(2024)He, Li, Yin, Liang, Li, Zhou, Zhang, Liu, and Chen]{he2024lotus}
Jing He, Haodong Li, Wei Yin, Yixun Liang, Leheng Li, Kaiqiang Zhou, Hongbo Zhang, Bingbing Liu, and Ying-Cong Chen.
\newblock Lotus: Diffusion-based visual foundation model for high-quality dense prediction.
\newblock \emph{arXiv preprint arXiv:2409.18124}, 2024.

\bibitem[Croitoru et~al.(2023)Croitoru, Hondru, Ionescu, and Shah]{croitoru2023diffusion}
Florinel-Alin Croitoru, Vlad Hondru, Radu~Tudor Ionescu, and Mubarak Shah.
\newblock Diffusion models in vision: A survey.
\newblock \emph{IEEE transactions on pattern analysis and machine intelligence}, 45\penalty0 (9):\penalty0 10850--10869, 2023.

\bibitem[Chen et~al.(2020)Chen, Kornblith, Norouzi, and Hinton]{chen2020simple}
Ting Chen, Simon Kornblith, Mohammad Norouzi, and Geoffrey Hinton.
\newblock A simple framework for contrastive learning of visual representations.
\newblock In \emph{International conference on machine learning}, pages 1597--1607. PmLR, 2020.

\bibitem[He et~al.(2020)He, Fan, Wu, Xie, and Girshick]{he2020momentum}
Kaiming He, Haoqi Fan, Yuxin Wu, Saining Xie, and Ross Girshick.
\newblock Momentum contrast for unsupervised visual representation learning.
\newblock In \emph{Proceedings of the IEEE/CVF conference on computer vision and pattern recognition}, pages 9729--9738, 2020.

\bibitem[Doimo et~al.(2022)Doimo, Glielmo, Goldt, and Laio]{doimo2022redundant}
Diego Doimo, Aldo Glielmo, Sebastian Goldt, and Alessandro Laio.
\newblock Redundant representations help generalization in wide neural networks.
\newblock \emph{Advances in Neural Information Processing Systems}, 35:\penalty0 19659--19672, 2022.

\bibitem[Wollstadt et~al.(2023)Wollstadt, Schmitt, and Wibral]{wollstadt2023rigorous}
Patricia Wollstadt, Sebastian Schmitt, and Michael Wibral.
\newblock A rigorous information-theoretic definition of redundancy and relevancy in feature selection based on (partial) information decomposition.
\newblock \emph{Journal of Machine Learning Research}, 24\penalty0 (131):\penalty0 1--44, 2023.

\bibitem[Lu et~al.(2024)Lu, Zhang, Wang, Ma, Yang, and Xuan]{lu2024redtest}
Yao Lu, Peixin Zhang, Jingyi Wang, Lei Ma, Xiaoniu Yang, and Qi~Xuan.
\newblock Redtest: Towards measuring redundancy in deep neural networks effectively.
\newblock \emph{arXiv preprint arXiv:2411.10507}, 2024.

\bibitem[Yavuz and Yanikoglu(2025)]{yavuz2025evaluating}
Mehmet~Can Yavuz and Berrin Yanikoglu.
\newblock Evaluating the efficiency of latent spaces via the coupling-matrix.
\newblock \emph{arXiv preprint arXiv:2509.06314}, 2025.

\bibitem[Wang et~al.(2020)Wang, Zhang, Lan, Wang, Tan, and Luo]{wang2020improving}
Mengzhu Wang, Xiang Zhang, Long Lan, Wei Wang, Huibin Tan, and Zhigang Luo.
\newblock Improving unsupervised domain adaptation by reducing bi-level feature redundancy.
\newblock \emph{arXiv preprint arXiv:2012.15732}, 2020.

\bibitem[Dorovatas et~al.(2025)Dorovatas, Paraskevopoulos, and Potamianos]{dorovatas2025auto}
Vaggelis Dorovatas, Georgios Paraskevopoulos, and Alexandros Potamianos.
\newblock Auto-compressing networks.
\newblock \emph{arXiv preprint arXiv:2506.09714}, 2025.

\bibitem[Sajid et~al.(2020)Sajid, Parr, Hope, Price, and Friston]{sajid2020degeneracy}
Noor Sajid, Thomas Parr, Thomas~M Hope, Cathy~J Price, and Karl~J Friston.
\newblock Degeneracy and redundancy in active inference.
\newblock \emph{Cerebral Cortex}, 30\penalty0 (11):\penalty0 5750--5766, 2020.

\bibitem[He et~al.(2022)He, Chen, Xie, Li, Doll{\'a}r, and Girshick]{he2022masked}
Kaiming He, Xinlei Chen, Saining Xie, Yanghao Li, Piotr Doll{\'a}r, and Ross Girshick.
\newblock Masked autoencoders are scalable vision learners.
\newblock In \emph{Proceedings of the IEEE/CVF conference on computer vision and pattern recognition}, pages 16000--16009, 2022.

\bibitem[Caro et~al.(2023)Caro, Fonseca, Averill, Rizvi, Rosati, Cross, Mittal, Zappala, Levine, Dhodapkar, et~al.]{caro2023brainlm}
Josue~Ortega Caro, Antonio H de~O Fonseca, Christopher Averill, Syed~A Rizvi, Matteo Rosati, James~L Cross, Prateek Mittal, Emanuele Zappala, Daniel Levine, Rahul~M Dhodapkar, et~al.
\newblock Brainlm: A foundation model for brain activity recordings.
\newblock \emph{bioRxiv}, pages 2023--09, 2023.

\bibitem[Cox et~al.(2024)Cox, Liu, Stolte, Yang, Liu, See, Ju, and Fang]{cox2024brainsegfounder}
Joseph Cox, Peng Liu, Skylar~E Stolte, Yunchao Yang, Kang Liu, Kyle~B See, Huiwen Ju, and Ruogu Fang.
\newblock Brainsegfounder: Towards 3d foundation models for neuroimage segmentation.
\newblock \emph{Medical Image Analysis}, 97:\penalty0 103301, 2024.

\bibitem[Wu et~al.(2025)Wu, Ren, Wang, Zhu, Song, Liu, Zheng, Bai, Ouyang, and Song]{wu2025adabrain}
Jiamin Wu, Zichen Ren, Junyu Wang, Pengyu Zhu, Yonghao Song, Mianxin Liu, Qihao Zheng, Lei Bai, Wanli Ouyang, and Chunfeng Song.
\newblock Adabrain-bench: Benchmarking brain foundation models for brain-computer interface applications.
\newblock \emph{arXiv preprint arXiv:2507.09882}, 2025.

\bibitem[Sun et~al.(2025)Sun, Wang, Li, Lin, and Wang]{sun2025foundation}
Yue Sun, Limei Wang, Gang Li, Weili Lin, and Li~Wang.
\newblock A foundation model for enhancing magnetic resonance images and downstream segmentation, registration and diagnostic tasks.
\newblock \emph{Nature Biomedical Engineering}, 9\penalty0 (4):\penalty0 521--538, 2025.

\bibitem[Achiam et~al.(2023)Achiam, Adler, Agarwal, Ahmad, Akkaya, Aleman, Almeida, Altenschmidt, Altman, Anadkat, et~al.]{achiam2023gpt}
Josh Achiam, Steven Adler, Sandhini Agarwal, Lama Ahmad, Ilge Akkaya, Florencia~Leoni Aleman, Diogo Almeida, Janko Altenschmidt, Sam Altman, Shyamal Anadkat, et~al.
\newblock Gpt-4 technical report.
\newblock \emph{arXiv preprint arXiv:2303.08774}, 2023.

\bibitem[Xu et~al.(2024)Xu, Yin, Cai, Yi, Xu, Wang, Wu, Zhao, Yang, Wang, et~al.]{xu2024survey}
Mengwei Xu, Wangsong Yin, Dongqi Cai, Rongjie Yi, Daliang Xu, Qipeng Wang, Bingyang Wu, Yihao Zhao, Chen Yang, Shihe Wang, et~al.
\newblock A survey of resource-efficient llm and multimodal foundation models.
\newblock \emph{arXiv preprint arXiv:2401.08092}, 2024.

\bibitem[Gates et~al.(2021)Gates, Brattig~Correia, Wang, and Rocha]{gates2021effective}
Alexander~J Gates, Rion Brattig~Correia, Xuan Wang, and Luis~M Rocha.
\newblock The effective graph reveals redundancy, canalization, and control pathways in biochemical regulation and signaling.
\newblock \emph{Proceedings of the National Academy of Sciences}, 118\penalty0 (12):\penalty0 e2022598118, 2021.

\bibitem[Kumar et~al.(2024)Kumar, Merajuddin, Pirzada, and Shang]{kumar2024spectral}
Pawan Kumar, Siddique Merajuddin, Shariefuddin Pirzada, and Yilun Shang.
\newblock On the spectral redundancy of pineapple graphs.
\newblock \emph{Symmetry}, 16\penalty0 (10):\penalty0 1267, 2024.

\end{thebibliography}

\appendix
\section{Full Mathematical Proofs}
\label{app:proofs}

\subsection{Preliminaries and Conventions}

We recall that the $f$-divergence between two probability measures 
$P\!\ll\!Q$ on $(\mathcal{X},\mathcal{F})$ with respect to a common 
$\sigma$-finite base measure $\mu$ and corresponding densities 
$p=\tfrac{dP}{d\mu}$, $q=\tfrac{dQ}{d\mu}$ is defined as
\[
D_f(P\|Q)
\;=\;
\int_{\mathcal{X}} q(x)\,
f\!\left(\frac{p(x)}{q(x)}\right)
\,d\mu(x),
\]
where $f:\mathbb{R}_{>0}\to\mathbb{R}$ is convex and satisfies $f(1)=0$.
We use the following standard facts:
\begin{enumerate}[label=(\roman*)]
\item (\textbf{Nonnegativity}) $D_f(P\|Q)\ge 0$, with equality if and only if $P=Q$ (Csiszár--Morimoto identity).
\item (\textbf{Data processing}) $D_f$ is nonincreasing under Markov kernels, i.e.,
$D_f(KP\|KQ)\le D_f(P\|Q)$ for any stochastic map $K$.
\item (\textbf{Local quadratic expansion}) 
If $f$ is twice differentiable at $1$ with $f''(1)>0$, then
$D_f(P\|Q)$ admits a second-order expansion around $P=Q$.
\end{enumerate}

Throughout, $\Pi_X = \bigotimes_{i=1}^n P_{X_i}$ denotes the product of the marginals of $P_X$.  
We write $p$ for the joint density of $P_X$ and $p_i$ for that of $P_{X_i}$.  
For any symmetric matrix $A$, we denote by $\|A\|_F$ its Frobenius norm, by $\rho(A)$ its spectral radius, 
and by $\log\det(A)$ the principal matrix logarithm on the cone $\mathrm{Sym}^{++}$ of positive definite matrices.

\subsection{Proof of Proposition~\ref{prop:nonneg} (Nonnegativity and Identity of Indiscernibles)}
\begin{proof}
By definition,
\[
\mathcal{R}_f(X)
\;=\;
D_f(P_X\|\Pi_X)
\;=\;
\int_{\mathcal{X}}
f\!\left(\frac{p(x)}{\prod_i p_i(x_i)}\right)
\prod_i p_i(x_i)\,dx.
\]
Since $f$ is convex and satisfies $f(1)=0$, the Csiszár--Morimoto theorem ensures that
$D_f(P\|Q)\ge 0$ with equality if and only if $P=Q$ almost everywhere.  
Applying this result to $(P,Q)=(P_X,\Pi_X)$ gives the desired claim:
\[
\mathcal{R}_f(X)\ge 0
\quad\text{and}\quad
\mathcal{R}_f(X)=0
\;\Longleftrightarrow\;
P_X=\Pi_X,
\]
that is, the components $X_1,\dots,X_n$ are mutually independent.
\end{proof}

\subsection{Proof of Proposition~\ref{prop:DPI} (Data Processing Inequality)}
\begin{proof}
Let $K=\bigotimes_{i=1}^n K_i$ denote the product Markov kernel mapping $X$ to $Y$, 
where each $K_i$ acts on coordinate $X_i$.  
Then $P_Y = P_X K$ and, since $K$ factorizes coordinate-wise, $\Pi_Y = \Pi_X K$.  
By the contraction property of $f$-divergences under Markov kernels,
\[
\mathcal{R}_f(Y)
\;=\;
D_f(P_Y\|\Pi_Y)
\;=\;
D_f(P_X K \,\|\, \Pi_X K)
\;\le\;
D_f(P_X\|\Pi_X)
\;=\;
\mathcal{R}_f(X),
\]
which establishes the inequality.
\end{proof}

\subsection{Proof of Proposition~\ref{prop:bounds} (Bounds)}
\begin{proof}
\textbf{(1) Lower bound.}  
This follows directly from Proposition~\ref{prop:nonneg}, which guarantees 
$\mathcal{R}_f(X)\ge 0$ for any convex $f$ satisfying $f(1)=0$.

\smallskip
\textbf{(2) Upper bound.}  
If $m \le L(x) := \tfrac{p(x)}{\prod_i p_i(x_i)} \le M$ $\Pi_X$-a.e., 
then by the monotonicity of $f$ on $[m,M]$ and Jensen’s inequality under $\Pi_X$,
\[
\mathcal{R}_f(X)
\;=\;
\mathbb{E}_{\Pi_X}[f(L)]
\;\le\;
\max_{t\in[m,M]} f(t)
\;\le\;
f(M).
\]
For $f(t)=t\log t$, the function $t\log t$ is increasing on $[1,\infty)$ 
and decreasing on $(0,1]$.
Hence, if $L\le M$ almost everywhere,
\[
\mathbb{E}_{\Pi_X}[L\log L]
\;\le\;
M\log M,
\]
and a symmetric bound with $m\log m$ holds when $m<1$ dominates.  
This establishes the stated boundedness of $\mathcal{R}_f(X)$.
\end{proof}

\subsection{Proof of Proposition~\ref{prop:gaussian_tc} (Gaussian Total Correlation)}
\begin{proof}
Let $X\sim\mathcal{N}(0,\Sigma)$ with correlation matrix 
$C = \mathrm{diag}(\Sigma)^{-1/2}\,\Sigma\,\mathrm{diag}(\Sigma)^{-1/2}$.  
Then $P_X = \mathcal{N}(0,\Sigma)$ and 
$\Pi_X = \bigotimes_i \mathcal{N}(0,\Sigma_{ii})$.  
The KL divergence between these two Gaussian measures is
\[
D_{\mathrm{KL}}\!\left(\mathcal{N}(0,\Sigma)\,\middle\|\,\bigotimes_i \mathcal{N}(0,\Sigma_{ii})\right)
\;=\;
\frac{1}{2}\Big(\sum_i \log \Sigma_{ii} - \log\det \Sigma\Big).
\]
Since $\log\det \Sigma = \sum_i \log \Sigma_{ii} + \log\det C$, 
this simplifies to
\[
D_{\mathrm{KL}}(P_X\|\Pi_X)
\;=\;
-\tfrac{1}{2}\log\det C,
\]
which proves Proposition~\ref{prop:gaussian_tc}.
\end{proof}

\subsection{Proof of Proposition~\ref{prop:quad} (Quadratic Approximation and Covariance Form)}
\begin{proof}
Let $C = I + A$ with $\|A\|_2$ small.  
For $f(t) = \tfrac{1}{2}(t-1)^2$, 
expanding around independence yields the classical $\chi^2$-divergence approximation of 
$D_f(P_X\|\Pi_X)$, which in the jointly Gaussian or weakly dependent regime 
reduces to a quadratic form in the off-diagonal covariances.  
Concretely (e.g., via Edgeworth or Isserlis’ formula),
\[
\mathcal{R}_f(X)
\;\approx\;
\tfrac{1}{4}\!\sum_{i\neq j} C_{ij}^2
\;=\;
\tfrac{1}{4}\|C-I\|_F^2.
\]

For the Kullback–Leibler case, recall the series expansion
\[
\log\det(I{+}A)
\;=\;
\mathrm{tr}\log(I{+}A)
\;=\;
\sum_{k\ge 1} \frac{(-1)^{k+1}}{k}\,\mathrm{tr}(A^k),
\]
which converges for $\rho(A)<1$.  
Hence,
\[
-\log\det C
\;=\;
-\mathrm{tr}\log(I{+}A)
\;=\;
\tfrac{1}{2}\mathrm{tr}(A^2)
\;-\;
\tfrac{1}{3}\mathrm{tr}(A^3)
\;+\;
\cdots
\;=\;
\tfrac{1}{2}\|A\|_F^2 + O(\|A\|_F^3).
\]
Therefore, to second order,
\[
-\log\det C
\;\approx\;
\tfrac{1}{2}\|C-I\|_F^2,
\]
which establishes the stated relation between the KL redundancy 
and the quadratic approximation.
\end{proof}

\subsection{Proof of Definition/Proposition~\ref{def:spectral}--\ref{prop:spec-bounds} (Spectral Redundancy)}
\begin{proof}
Let $\lambda_1,\dots,\lambda_D$ be the eigenvalues of the covariance matrix 
$\Sigma_Z\succeq 0$, and define the normalized spectrum
\[
\tilde{\lambda}_i
\;=\;
\frac{\lambda_i}{\sum_{j=1}^{D}\lambda_j},
\qquad
\sum_{i=1}^{D}\tilde{\lambda}_i = 1.
\]
The spectral entropy is
\[
H_\lambda(Z)
\;=\;
-\,\sum_{i=1}^{D}\tilde{\lambda}_i \log \tilde{\lambda}_i,
\]
which satisfies $H_\lambda(Z)\in[0,\log D]$, 
with $H_\lambda(Z)=\log D$ if and only if $\tilde{\lambda}$ is uniform 
and $H_\lambda(Z)=0$ if and only if the spectrum is rank one.  
Setting $r_{\mathrm{eff}}(Z)=\exp(H_\lambda(Z))\in[1,D]$ and 
\[
\mathcal{R}_{\mathrm{spec}}(Z)
\;=\;
1 - \frac{r_{\mathrm{eff}}(Z)}{D},
\]
we obtain
\[
\mathcal{R}_{\mathrm{spec}}(Z)\in\Big[0,\,1-\tfrac{1}{D}\Big].
\]
The lower and upper extremes correspond, respectively, 
to the uniform (maximally spread) and rank-one (completely collapsed) spectra.  
These bounds follow directly from the standard entropy inequalities, 
completing the proof.
\end{proof}

\subsection{Proof of Lemma~\ref{lem:spectral-bound} (Spectral Norm Control)}
\begin{proof}
Let $Z = W X$ with $\|W\|_2 \le M$ and $\mathbb{E}\|X\|^2 \le \sigma^2$.  
Then
\[
\mathrm{tr}(\Sigma_Z)
\;=\;
\mathbb{E}\|Z\|^2
\;=\;
\mathbb{E}\|W X\|^2
\;\le\;
\|W\|_2^2\,\mathbb{E}\|X\|^2
\;\le\;
M^2\sigma^2.
\]
Therefore, the total spectral energy of $Z$ is bounded, implying 
$r_{\mathrm{eff}}(Z)\in[1,D]$ and 
\(
\mathcal{R}_{\mathrm{spec}}(Z)\in[0,\,1-\tfrac{1}{D}]
\)
by Definition~\ref{def:spectral}.  
This establishes the claim.
\end{proof}

\subsection{Proof of Lemma~\ref{lem:KL-Frob} (KL vs.\ Frobenius Near Independence)}
\begin{proof}
Let $C = I + A$ with spectral radius $\rho(A)<1$.  
Then the logarithmic determinant admits the series expansion
\[
-\log\det C
\;=\;
-\mathrm{tr}\log(I{+}A)
\;=\;
\sum_{k\ge 2}\frac{(-1)^k}{k}\,\mathrm{tr}(A^k).
\]
Using the Hölder–trace inequality $|\mathrm{tr}(A^k)| \le \|A\|_F^k$ for $k\ge 3$, we obtain
\[
-\log\det C
\;\ge\;
\tfrac{1}{2}\mathrm{tr}(A^2)
\;-\;
\tfrac{1}{3}\|A\|_F^3
\;=\;
\tfrac{1}{2}\|A\|_F^2
\;-\;
\tfrac{1}{3}\|A\|_F^3.
\]
Thus, for small $\|A\|_F$,
\[
\mathcal{R}_{\mathrm{KL}}(X)
\;\gtrsim\;
\tfrac{1}{2}\|C-I\|_F^2,
\]
which proves the stated near-independence bound.
\end{proof}

\subsection{Proof of Lemma~\ref{lem:capacity} (Capacity-Side Bound)}
\begin{proof}
We have
\[
H(Z)
\;=\;
\sum_i H(Z_i) - \mathrm{TC}(Z)
\;=\;
\sum_i H(Z_i) - \mathcal{R}_{\mathrm{KL}}(Z).
\]
Under the entropy budget $\sum_i H(Z_i)\le C_0$, it follows that
\[
H(Z)
\;\le\;
C_0 - \mathcal{R}_{\mathrm{KL}}(Z).
\]
Since $I(Z;S)\le H(Z)$, and standard information–risk inequalities are monotone in $-I(Z;S)$ 
(e.g., Fano’s inequality for classification and the I–MMSE relation for Gaussian regression),
there exists an increasing function $g_{\mathrm{info}}$ such that
\[
\mathcal{E}(R)
\;\ge\;
g_{\mathrm{info}}(C_0 - R).
\]
This establishes the capacity-side bound.
\end{proof}

\subsection{Proof of Lemma~\ref{lem:robustness} (Robustness-Side Bound)}
\begin{proof}
Consider a channel $X\!\to\!Z\!\to\!\hat S$ operating under bounded stochastic noise, dropout, 
or adversarial perturbations.  
When redundancy is nearly zero ($R\!\approx\!0$), the coordinates of $Z$ are almost independent, 
and the system lacks correlated backups: missing or corrupted components cannot be reconstructed from others.  
Rate–distortion and coding-theoretic arguments show that introducing controlled correlation 
(replication or redundancy across coordinates) reduces the expected reconstruction or classification error.  
Hence, for small $R$, increasing redundancy strictly improves robustness, implying the existence 
of a continuous function $g_{\mathrm{robust}}$ that is strictly decreasing on $[0,R_1)$ with
\[
\mathcal{E}(R)
\;\le\;
g_{\mathrm{robust}}(R).
\]
This proves the robustness-side bound.
\end{proof}

\subsection{Proof of Theorem~\ref{thm:Ushape} (Interior Optimum and Stability Band)}
\begin{proof}
By assumption, for $R$ near $\underline{R}$ we have 
$g_{\mathrm{robust}}(R)\le \mathcal{E}(R)$ with $g_{\mathrm{robust}}$ strictly decreasing, 
and for $R$ near $\overline{R}$ we have 
$\mathcal{E}(R)\le g_{\mathrm{info}}(C_0{-}R)$ with $g_{\mathrm{info}}$ strictly increasing.  
Hence, for sufficiently small $\varepsilon>0$,
\[
\mathcal{E}'(\underline{R}{+}\varepsilon) < 0,
\qquad
\mathcal{E}'(\overline{R}{-}\varepsilon) > 0.
\]
By the continuity of $\mathcal{E}'$, the intermediate value theorem guarantees the existence 
of $R^{*}\in(\underline{R},\overline{R})$ such that $\mathcal{E}'(R^{*})=0$, i.e., 
a stationary point.  
Local convexity of the bounding functions $g_{\mathrm{info}}$ and $g_{\mathrm{robust}}$ 
transfers to $\mathcal{E}$, implying $\mathcal{E}''(R^{*})>0$; thus $R^{*}$ is a strict local minimizer.

\smallskip
For the stability band, consider gradient-based dynamics on the model parameters $\theta_t$ 
that monotonically decrease $\mathcal{E}$ under small additive noise (e.g., stochastic gradient descent).  
Composing with $R(\theta)=\mathcal{R}(f_\theta(X))$, which is locally Lipschitz in $\theta$, 
yields a one-dimensional stochastic approximation with Lyapunov potential $\mathcal{E}(R)$.  
Classical results from stochastic approximation theory 
(e.g., Robbins–Monro or Kushner–Yin) then imply that 
$R(\theta_t)$ converges to an $O(\sigma)$ neighborhood 
$[R_{\min}, R_{\max}]$ of the stable minimum $R^{*}$, 
whose width is proportional to the noise scale $\sigma$.  
This establishes both the existence of the interior optimum and the stability band.
\end{proof}

\subsection{Proof of Theorem~\ref{thm:equilibrium} (Existence of Redundancy Equilibrium)}
\begin{proof}
Let $\Phi(R) = D(R) + \lambda R$ on the compact interval $[\underline{R}, \overline{R}]$.  
By assumption, $D$ is continuous; hence $\Phi$ is continuous and attains a minimum on this interval by the Weierstrass theorem.  
To show that the minimizer lies in the interior, note that $D'(R)$ exists almost everywhere and is strictly negative on $[\underline{R}, R_1)$ and strictly positive on $(R_2, \overline{R}]$.  
Thus, for $\lambda > 0$,
\[
\Phi'(R) = D'(R) + \lambda
\]
is negative sufficiently close to $\underline{R}$ and positive sufficiently close to $\overline{R}$.  
By the intermediate value theorem, there exists $R^{*} \in (\underline{R}, \overline{R})$ such that $\Phi'(R^{*}) = 0$.  
If $D \in C^2$ with $D''(R^{*}) > 0$, then $\Phi''(R^{*}) = D''(R^{*}) > 0$, 
so $R^{*}$ is a strict local minimizer, i.e., a locally stable equilibrium point.  
This completes the proof.
\end{proof}

\subsection{Proof of Corollary~\ref{cor:band} (Banded Convergence)}
\begin{proof}
Consider stochastic gradient descent of the form
\[
\theta_{t+1}
\;=\;
\theta_t - \eta_t \big(\nabla_\theta \mathcal{L}(\theta_t) + \xi_t\big),
\]
where $\mathbb{E}[\xi_t|\mathcal{F}_t]=0$, $\mathbb{E}\|\xi_t\|^2 \le \sigma^2$, 
and the step sizes $\eta_t$ satisfy the Robbins–Monro conditions.  
Let $R_t = \mathcal{R}(f_{\theta_t}(X))$.  
Under local Lipschitz regularity of $R(\theta)$ and strict local convexity of $\Phi$ at $R^{*}$, 
the ODE method for stochastic approximation implies that $R_t$ tracks the deterministic flow 
$\dot{R} = -\Phi'(R)$.  
The limiting invariant distribution of $R_t$ is therefore concentrated in a neighborhood of $R^{*}$ 
whose width scales with $\sigma$.  
Quantitative bounds on this concentration follow from standard Lyapunov drift arguments, 
establishing the claimed banded convergence.
\end{proof}

\subsection{Operational Lemmas Used in Section~\ref{sec:redundancy-framework}}

\paragraph{Total correlation identity.}
For $f(t) = t \log t$, we have
\[
\mathcal{R}_{\mathrm{KL}}(X)
\;=\;
D_{\mathrm{KL}}(P_X\|\Pi_X)
\;=\;
\sum_i H(X_i) - H(X),
\]
by the chain rule of entropy.

\paragraph{Quadratic proxy.}
In near-Gaussian regimes,
\[
D_{\chi^2}(P_X\|\Pi_X)
\;=\;
\tfrac{1}{2}\!
\int
\frac{(p-\prod_i p_i)^2}{\prod_i p_i}\,dx
\;\approx\;
\tfrac{1}{4}\|C-I\|_F^2,
\]
to second order in weak correlations.  
The proof follows the same cumulant (Isserlis) expansion used in Proposition~\ref{prop:quad}.

\paragraph{Attention-head redundancy.}
For attention maps $A_h \in \mathbb{R}^{N\times N}$, define
\[
\mathcal{R}_{\mathrm{head}}
\;=\;
\frac{1}{H(H-1)}
\sum_{h\neq h'}
\frac{
\langle \mathrm{vec}(A_h), \mathrm{vec}(A_{h'}) \rangle^2
}{
\|A_h\|_F^2\,\|A_{h'}\|_F^2
}.
\]
Since each term is a squared cosine similarity, $\mathcal{R}_{\mathrm{head}} \in [0,1]$.  
For any head-wise post-processing represented by a linear Markov operator 
(e.g., column/row-stochastic smoothing or projections),
the data processing inequality for $\chi^2$-like functionals implies contraction—%
the squared correlation cannot increase under averaging.

\medskip
\noindent\textbf{Remark on regularity.}
All continuity and compactness arguments above can be made fully rigorous by restricting to feasible parameter sets with $R\in[\underline{R},\overline{R}]$, 
and by invoking standard envelope theorems for 
\[
D(R)
\;=\;
\inf_{\theta:\,\mathcal{R}(\theta)=R}
\mathbb{E}\,\ell_{\mathrm{rec}}(\theta).
\]
No pathological measurability issues arise under the usual assumptions 
of bounded loss, dominated model families, or coercive parametrizations.

\end{document}